\newif\ifarxiv
\arxivfalse 
\arxivtrue 
 \ifarxiv \documentclass{article} 
 \usepackage[margin=1in]{geometry} 
 \usepackage{natbib}
 \usepackage{hyperref}
 \usepackage{graphicx}
 \usepackage{xcolor}
 \else 
\documentclass[letterpaper]{article} 
\usepackage{aaai24}  

\usepackage{times}  
\usepackage{helvet}  
\usepackage{courier}  
\usepackage[hyphens]{url}  
\usepackage{graphicx} 
\usepackage{natbib}  
\usepackage{caption} 
\fi

\usepackage{algorithm}
\usepackage{algorithmic}
\usepackage{amsthm}  
\usepackage{amsmath}

\usepackage{csquotes}
\usepackage{thmtools,thm-restate}

\newcommand{\p}[1]{\left( #1 \right)}

\newcommand{\cd}[0]{\cdot}

\newtheorem*{theorem*}{Theorem}

\newtheorem*{lemma*}{Lemma}
\newtheorem*{corollary}{Corollary}

\newtheorem{defin}{Definition}

\newtheorem{example}{Example}

\newcommand{\nitem}[0]{\ensuremath{n}}
\newcommand{\npresent}[0]{\ensuremath{k}}

\newcommand{\Human}[0]{\ensuremath{H}}
\newcommand{\Alg}[0]{\ensuremath{A}}

\newcommand{\algdist}[0]{\ensuremath{\mathcal{D}^a}}
\newcommand{\humdist}[0]{\ensuremath{\mathcal{D}^h}}
\newcommand{\var}[0]{\ensuremath{\sigma^2}}

\ifarxiv \else 
\usepackage{newfloat}
\usepackage{listings}
\DeclareCaptionStyle{ruled}{labelfont=normalfont,labelsep=colon,strut=off} 
\floatstyle{ruled}
\newfloat{listing}{tb}{lst}{}
\floatname{listing}{Listing}
\fi 
\ifarxiv \else 
\fi

\ifarxiv \else  \fi 

\ifarxiv
\usepackage{authblk}
\title{When Are Two Lists Better than One?:\\  Benefits and Harms in Joint Decision-making
\footnote{Non-archival submission: A version of this work previously appeared at AAAI '24 (selected for oral presentation). No published proceedings for AAAI '24 available at time of FORC submission (ArXiv version at \url{https://arxiv.org/abs/2308.11721}).}
}
\author[1]{Kate Donahue\footnote{Work done while at Google.}\thanks{kdonahue@cs.cornell.edu}}
\author[2]{Sreenivas Gollapudi\thanks{sgollapu@google.com}}
\author[2]{Kostas Kollias\thanks{kostaskollias@google.com }}
\affil[1]{Cornell University}
\affil[2]{Google}
\date{}
\else 

\title{When are two lists better than one?:  benefits and harms in joint decision-making
\author {
    Kate Donahue\textsuperscript{\rm 1}\footnote{Work done while at Google.},
    Sreenivas Gollapudi\textsuperscript{\rm 2},
    Kostas Kollias\textsuperscript{\rm 2}
}
\affiliations {
    \textsuperscript{\rm 1}Cornell University\\
    \textsuperscript{\rm 2}Google\\
    kdonahue@cs.cornell.edu, sgollapu@google.com, kostaskollias@google.com
}
\fi

\begin{document}

\maketitle

\begin{abstract}
Historically, much of machine learning research has focused on the performance of the algorithm alone, but recently more attention has been focused on optimizing joint human-algorithm performance. Here, we analyze a specific type of human-algorithm collaboration where the algorithm has access to a set of $n$ items, and presents a subset of size $k$ to the human, who selects a final item from among those $k$. This scenario could model content recommendation, route planning, or any type of labeling task. Because both the human and algorithm have imperfect, noisy information about the true ordering of items, the key question is: which value of $k$ maximizes the probability that the best item will be ultimately selected? For $k=1$, performance is optimized by the algorithm acting alone, and for $k=n$ it is optimized by the human acting alone. 
Surprisingly, we show that for multiple noise models, it is optimal to set $k \in [2, n-1]$ - that is, there are strict benefits to collaborating, even when the human and algorithm have equal accuracy separately. We demonstrate this theoretically for the Mallows model and experimentally for the Random Utilities models of noisy permutations. However, we show this pattern is \emph{reversed} when the human is anchored on the algorithm's presented ordering -  the joint system always has strictly worse performance. We extend these results to the case where the human and algorithm differ in their accuracy levels, showing that there always exist regimes where a more accurate agent would strictly benefit from collaborating with a less accurate one, but these regimes are asymmetric between the human and the algorithm's accuracy.

\end{abstract}

\section{Introduction}

Consider the following motivating example: 
\begin{displayquote}
Alice is a doctor trying to classify a medical record with one of $\nitem$ different labels. Based on her professional expertise and relevant medical information she has access to, she is able to make some ranking over which of these labels is most likely to be accurate. However, she is not perfect, and sometimes picks the wrong label. She decides to use a machine learning algorithm as a tool to assist in picking the right label. The algorithm similarly has a goal of maximizing the probability of picking the correct label. However, the algorithm and Alice rely on somewhat different information sources in making their predictions: vast troves of data for the algorithm, and personal conversations with the patient for the human, for example. Because of this, their rankings over the true labels will often differ slightly. The algorithm communicates its knowledge by presenting its top $\npresent$ labels to Alice, who picks her top label among those that are presented.
For what settings and what values of $\npresent$ will Alice and the algorithm working together have a higher chance of picking the right label? 
\end{displayquote} 

While this example is in the medical domain, in human-algorithm collaboration more generally, often the algorithm can provide assistance, but the human makes the final decision. This is the case in other settings as well: a diner trying to find the best restaurant, a driver trying to find the best route, or a teacher trying to find the best pedagogical method. This framework requires a shift in thinking: rather than focus on optimizing the performance of the algorithm alone, the goal is to build an algorithm that maximizes the performance of the human-algorithm system.

If the algorithm were able to tell Alice (the human) exactly which label she should pick ($\npresent=1$), then this problem would simply reduce to that of building a highly accurate machine learning system. However, in the medical prediction setting, it is unrealistic to assume that the algorithm can force Alice to pick a particular label. If the algorithm presented all of the items to Alice ($\npresent = \nitem$), then this would be equivalent to Alice solving the task herself. In the case where $\nitem$ is large, considering each possible label may be infeasible. However, even if Alice could consider all $\nitem$ items herself, we will show that there are often settings where allowing the algorithm to narrow the set of items to $\npresent$ strictly increases the probability of picking the correct item. 

In this paper, we will focus on the role of the noise distributions that govern the human and algorithm and how \emph{independent} these distributions are. In particular, we will be interested in how strongly the human's permutation is affected by the algorithm's prediction, or the strength of \emph{anchoring}. In this paper, we will explore different models of noisy predictions, and give theoretical and empirical results describing when the joint human-algorithm system has a higher chance of picking the best item. 

We begin in Section \ref{sec:relatedwork} by connecting our model to related works. In Section \ref{sec:model} we describe the theoretical model that we will explore and in Section \ref{sec:prelim} we introduce preliminary technical tools and results that we will need in later sections. Section \ref{sec:noanch} considers the case where the human and the algorithm have independent ordering (lack of anchoring), gives theoretical proofs for conditions where there are strict \emph{benefits} to using a joint human-algorithm system. Specifically, Section \ref{sec:equalacc} shows that strict benefits are guaranteed when both the algorithm and human have equal accuracy rates and the algorithm presents exactly 2 items to the human. Section \ref{sec:diffacc} explores the case where the human and algorithm can differ in their accuracy rates and there are exactly 3 items, of which 2 are presented to the human, showing that there is always a regime where a more accurate player can strictly improve their accuracy by joining with a less accurate partner, but this regime is asymmetric between the human and the algorithm: the human's accuracy rate is more impactful. Next, Section \ref{sec:anch} considers the case where the human may have cognitive biases that cause them to \emph{anchor} on the algorithm's ordering of items: we show that such anchoring makes the joint human-algorithm system have much worse performance. Section \ref{sec:rum} shows numerically that our theoretical results in the previous sections extend to a separate model of permutation distributions, suggesting that our results may have broader implications. Finally, Section \ref{sec:conclude} concludes and discusses avenues for future work. 

\section{Related work}\label{sec:relatedwork}

Studying human-algorithm collaboration is a large, rapidly-growing, and highly interdisciplinary area of research. 
Some veins of research are more ethnographic, studying how people use algorithmic input in their decision-making \cite{lebovitz2021ai, lebovitz2020incorporate, beede2020human, yang2018investigating, okolo2021cannot}. Other avenues work on developing ML tools designed to work with humans, such as in medical settings \cite{raghu2019algorithmic} or child welfare phone screenings \cite{chouldechova2018case}. Finally, and most closely related to this paper, some works develop theoretical models to analyze human-algorithm systems, such as \cite{rastogi2022unifying, cowgill2020algorithmic, bansal2021most, steyvers2022bayesian, madras2017predict}. \citet{bansal2021does} proposes the notion of \emph{complementarity}, which is achieved when a human-algorithm system together has performance that is strictly better than either the human or the algorithm could achieve along. \cite{steyvers2022bayesian} uses a Bayesian framework to model human-algorithmic complementarity, while \cite{donahue2022human} studies the interaction between complementarity and fairness in joint human-algorithm decision systems, and \cite{rastogi2022unifying} provides a taxonomy of how humans and algorithms might collaborate. \cite{kleinberg2021algorithmic} is structurally similar to ours in that it uses the Mallows model and RUM model to give theoretical guarantees for performance related to rankings of items. However, its setting is human-algorithm \emph{competition} rather than \emph{cooperation}, where the question is whether it is better to rely on an algorithmic tool or more noisy humans to rank job candidates.

One related area of research is \enquote{conformal prediction} where the goal is to optimize the subset that the algorithm presents to the human, such as in \cite{straitouri2022provably, wang2022improving, angelopoulos2020uncertainty, vovk2005algorithmic, babbar2022utility, straitouri2023designing}. This formulation is structurally similar to ours, but often takes a different approach (e.g. optimizing the subset given some prediction of how the human will pick among them). Another related area is \enquote{learning to defer}, where an algorithmic tool learns whether to allow a human (out of potentially multiple different humans) to make the final decision, or to make the prediction itself (e.g. \cite{hemmer2022forming, madras2017predict, raghu2019direct}). Finally, a third related area is multi-stage screening or pipelines, where each stage narrows down the set of items further (e.g. \cite{blum2022multi, wang2023uncertainty, dwork2020individual, bower2022random}). \cite{NEURIPS2021_162d1815} specifically studies the case with multiple imperfect nominators who each suggest an action to a ranker, who picks among them (and explores how to optimize this setting). 

Some papers study how humans rely on algorithmic predictions - for example, \cite{de2020case} empirically studies a real-life setting where the algorithm occasionally provided incorrect predictions and explores how the human decision-maker is able to overrule its predictions, while \cite{benz2023humanaligned} studies under what circumstances providing confidence scores helps humans to more accurately decide when to rely on algorithmic predictions. \cite{Bryce_Spiess} studies a case where the human decision-maker views the algorithm's recommendation as the \enquote{default} - similar to our \enquote{anchoring} setting, while \cite{vasconcelos2023explanations} studies how explanations can reduce the impact of anchoring, and  \cite{whogoesfirst} empirically studies the impact of anchoring in a medical setting. \cite{rambachan2021identifying} studies how to identify human errors in labels from observational data, while \cite{alur2023auditing} explores how an algorithmic system can detect when a human actor has access to different sources of information than the algorithm itself. Also in a medical setting, \cite{cabitza2021studying} studies how \enquote{interaction protocols} with doctors and algorithmic tools can affect overall accuracy. \cite{chen2023understanding} empirically explores how human rely on their intuition along with algorithmic explanations in making decisions. \cite{mozannar2023show} explores a setting where an LLM is making recommendations of code snippets to programmers, with the goal of making recommendations that are likely to be accepted. Related to complementarity, \cite{guszcza2022hybrid} describes the principles of  \enquote{hybrid intelligence} necessary for optimizing human-algorithm collaboration. 

There has also been a series of work looking more specifically at human-algorithm collaboration in bandit settings. \citet{gao2021human} learns from batched historical human data to develop an algorithm that assigns each task at test time to either itself or a human. \citet{chan2019assistive} studies a setting where the human is simultaneously learning which option is best for them. However, their framework allows the algorithm to overrule the human, which makes sense in many settings, but is not reasonable in some settings like as our motivating medical example. \citet{bordt2022bandit} formalizes the problem as a two-player setting where both the human and algorithm take actions that affect the reward both experience. \cite{Agarwal2022DiversifiedRF} and \cite{agarwal2023online} study the case where a \enquote{menu} of $\npresent$ arms out of $\nitem$ are presented to the human, who selects a final one based on a preference model. This setting differs from ours in the model of human preferences over items, as well as the goal of optimizing for the algorithm's overall regret. \cite{yao2023bad} studies a related setting where multiple content creators each recommend a top $k$ set of items to humans, who pick among those $k$ according to a RUM - key differences are that content creators are competing with each other and also learning their own utility functions over time. \cite{tian2023towards} considers the case where the human's mental model of the algorithm is changing over time, and models this as a dynamical system. 

Additionally, some work has used the framework of the human as the final decision-maker and studied how to disclose information so as to incentivize them to take the \enquote{right} action. \citet{immorlica2018incentivizing} studies how to match the best regret in a setting where myopic humans pull the final arm. \citet{hu2022incentivizing} studies a related problem with combinatorial bandits, where the goal is to select a subset of the total arms to pull. \citet{bastani2022learning} investigates a more applied setting where each human is a potential customer who will become disengaged and leave if they are suggested products (arms) that are a sufficiently poor fit. \citet{kannan2017fairness} looks at a similar model of sellers considering sequential clients, specifically investigating questions of fairness. In general, these works differ from ours in that they assume a new human arrives at each time step, and so the algorithm is able to selectively disclose information to them.

\section{Models and notation}\label{sec:model}
In this section, we formalize our theoretical model of how the human and algorithm interact, introduce key notation and assumptions, and describe permutation models that later sections will rely on. 
\subsection{Human-algorithm collaboration model}
We assume that there are $\nitem$ items $\{x_1, \ldots x_{\nitem}\}$, and that the goal is to pick item $x_1$. Each item could represent labels for categorical prediction, news articles of varying relevance, or driving directions with variable levels of traffic, for example. There are two actors: the first ($\Alg$) narrows the items from $\nitem$ total items to a top $\npresent<\nitem$ items which are presented to the second actor ($\Human$), which picks a single item among them. One consistent assumption we will make is that the second actor $\Human$ is not able to directly access or choose from the full set of items: this could be, for example, because $\npresent<< \nitem$ and $\Human$ is bandwidth-limited in how many items it can consider. This model is quite broad: the two actors could be interacting recommendation algorithms, for example, or sequential levels of decision-making among human committees. However, the motivating example we will focus on in this paper is when the first actor $\Alg$ is an algorithm and the second actor $\Human$ is a human. This setting naturally fits with the assumption that $\Human$ is bandwidth-limited, and also motivates the assumption that $\Alg$ and $\Human$ have differing orders for the items, drawn from potentially differing sources of knowledge, but are unable to directly communicate that knowledge to each other. This formulation also allows us to connect with the extensive literature on human-algorithm collaboration, which we discuss further in Section \ref{sec:relatedwork}. Throughout, we will be interested in when we can prove that complementarity (as defined in \cite{bansal2021does}) occurs: in our setting, that the probability of picking the best item $x_1$ is strictly greater in the combined setting than with either the human or algorithm alone.

We will use $\pi^a, \pi^h$ (and sometimes $\rho^a, \rho^h$) to denote the orderings of the algorithm and human over the $\nitem$ items, with $\pi_i^a=x_j$ meaning that the algorithm ranks item $x_j$ in the $i$th place. We will use $\pi_{[\npresent]}^a$ to denote the $\npresent$ items that the algorithm ranks first (and thus presents to the human) and $\pi_{-[\nitem-\npresent]}^a$ to denote the $\nitem-\npresent$ items that the algorithm ranks last (and fails to present to the human). Both $\pi^a, \pi^h$ are random variables drawn from distributions $\pi^a\sim \algdist, \pi^h \sim \humdist$. We will often refer to the joint human-algorithm system as the \emph{combined system}. 

The distributions $\algdist, \humdist$ may be independent: this could reflect the case where both the human and algorithm come up with orderings separately, and then the algorithm presents a set of items for the human to pick between, where the human picks the best item according to their previously-determined ranking. We refer to the case of independent orderings as the \emph{unanchored} case. Alternatively, the distributions $\algdist, \humdist$ may be correlated. In particular, we will compare the \emph{unanchored} case with that of \emph{anchored} ordering. In this setting, the algorithm draws an ordering $\pi^a \sim \algdist$, which then influences the human's permutation -- we will describe what this means technically for different noise models in the next section. This models settings where the algorithm presents a \emph{ordering} of items to the human, rather than a set, which biases the human to varying degrees.

\subsection{Assumptions}
One key assumption is the structure of the human-algorithm system: namely, the algorithm selects $\npresent$ items from which the human picks a final element. As mentioned previously, this could reflect settings the algorithm \emph{must} narrow the set: where the total set of items $\nitem$ is too large for the human to fully explore (e.g. the set of news articles, or the set of possible routes between two destinations). It could also reflect cases where the algorithm \emph{chooses} to narrow the set in order to express its knowledge. This structure also allows the human a structured, but wide range of flexibility. 

However, there are other potential models that could also be feasible. For example, it could be the case that both the human and algorithm present their permutations $\pi^a, \pi^h$, and the combined ranking $\pi^c$ is constructed by \enquote{voting} between each of the rankings, potentially with some uneven weighting between $\pi^a, \pi^h$ based on expertise. Another model could involve the human going down the algorithm's ranking $\pi^a$, stopping whenever it reaches its best item $\pi^h_1$ or with some probability $p$ after inspecting each item. A third model could involve iterative processes where the human and algorithm can refine their rankings through shared information. Note that many of these models would require the human to consider more than $\npresent$ items, which contradicts this model's consideration of a bandwidth-limited agent $\Human$. 

While all of these models could be interesting extensions to explore, in general they are more complicated than ours. Despite the relatively simple and natural structure of our human-algorithm system, we will show that it admits a rich structure with relatively clean results. 

\subsection{Noise models}\label{sec:noisemodels}
In this section, we introduce the noise models we will use for $\algdist, \humdist$, which governs how the algorithm and human respectively arrive at noisy permutations over each of the $\nitem$ items. Both of these noise models are standard in the literature, which is what prompted us to consider them in our paper. We will focus primarily on the Mallows model because of its theoretical tractability, but in Section \ref{sec:rum} we will demonstrate that our core phenomena extend numerically to the Random Utility Model. 
\subsubsection{Mallows model}
 The first is the Mallows model, which has been used extensively as a model of permutations \cite{mallows1957non}. The model has two components: a central ordering $\pi^*$ (here, assumed to be the \enquote{correct} ordering $\{x_1, x_2, \ldots x_{\nitem}\}$) and an accuracy parameter $\phi>0$, where higher $\phi$ means that the distribution more frequently returns orderings that are close to the central ordering $\pi^*$. The probability of any permutation $\pi$ occurring is given by
$\frac{1}{Z} \cd \exp\p{-\phi \cd d(\pi^*, \pi)}$
where $Z$ is a normalizing constant $\sum_{\pi' \in P} \exp\p{-\phi \cd d(\pi^*, \pi)}$ involving a sum over the set all permutations $P$ and $d(\pi^*, \pi)$ is a distance metric between permutations. In this work, we will use Kendall-Tau distance, which is standard. In particular, the Kendall-Tau distance is equivalent to the number of \emph{inversions} in $\pi$. An inversion occurs when element $x_i$ is ranked above $x_j$ in the true ordering $\pi^*$, but is ranked below $x_j$ in $\pi$. This can be roughly thought of as the number of \enquote{pairwise errors} $\pi$ makes in ordering each of the elements. 

In this paper, we will model \emph{anchoring} in the Mallows model through a parameter $w_a \in [0, 1]$ which reflects how strongly the human is influenced by the algorithm's realized permutation $\pi^a$. We model the distance of a particular permutation $\pi$ as given by the weighted average of the distance to the human's central distribution $\pi^*$ and the algorithm's realized permutation $\pi^a$: 
$$d(\pi^*, \pi^a, \pi, w_a) = (1-w_a) \cd d(\pi^*,\pi) + w_a \cd d(\pi^a, \pi)$$
For $w_a =0$, we recover the \emph{unanchored} case where the human draws their permutation from a Mallows distribution centered at the correct ordering $\mathcal{D}\p{\pi^* = \{x_1, x_2, \ldots x_{\nitem}\}}$. For $w_a=1$, the human takes the algorithm's presented ordering as the \enquote{true} ordering and draws permutations centered on it. 

\subsubsection{Random Utility model}

The Random Utility Model (RUM) has similarly been extensively used as a model of permutations \cite{thurstone1927law}. In this model, item $i$ has some true value $\mu_i$, where we assume $\mu_i$ is descending in $i$. The human and algorithm only have access to noisy estimates of these values, $\hat X_i^a \sim \mathcal{D}(\mu_i, \var)$ for some distribution $\mathcal{D}$ with variance $\var$ (often assumed to be Gaussian, which we will use in this paper). These noisy estimates are then used to produce an order $\pi^a, \pi^h$ in descending order of the values $\{\hat X_i^a\}, \{\hat X_i^h\}$. 

In RUM, we model anchoring through shifting the mean of element $i$ depending on the index $j$ that the algorithm ranked it in: 
$$\hat X_i^h \sim \mathcal{D}((1-w_a) \cd \mu_i + w_a \cd \mu_j, \var_h)$$ 
where $w_a$ is a weight parameter indicating how much the algorithm's ordering anchors the human's permutation, and $j$ is the index of item $i$ in the algorithm's permutation $\pi^a$. 

\section{Preliminary tools: mapping between good and bad events}\label{sec:prelim}
First, this section describes preliminary tools we will need in order to prove the results in later sections. Note that every result in this subsection holds for all distributions of human and algorithmic permutations $\humdist, \algdist$, and regardless of the level of anchoring. However, we will find these tools useful for analysis in later subsections with more specific assumptions on $\humdist, \algdist$. Throughout, our goal will be \emph{complementarity} as defined in \cite{bansal2021does}: when the joint system has a higher chance of picking the best item than either the human or algorithm alone. 

First, Definitions \ref{def:goodevent} defines \enquote{good events} where the joint human-algorithm system picks the best arm, where the algorithm alone would not have, and Definition \ref{def:badevent} defines\enquote{bad events}, where the joint system fails to pick the best arm, where the algorithm alone would have. Complementarity occurs whenever the total probability of \enquote{good events} is greater than the total probability of \enquote{bad events}. Note that these could be identically defined with respect to when the human would have picked the best arm. However, defining events relative to the algorithm will make later proofs technically simpler. 

\begin{defin}\label{def:goodevent}
A \enquote{good event} is a pair of permutations $\rho^a,\rho^h$ where the joint human-algorithm system selects the best arm $x_1$ when the algorithm alone would not have picked it. The \enquote{good event} occurs when in one of two cases holds: 
\begin{enumerate}
    \item The algorithm does not rank $x_1$ first but includes it in the $\npresent$ items it presents, while the human ranks item $x_1$ first ($\rho_1^a \not = x_1, x_1 \in \rho_{[\npresent]}^a, \rho_1^h = x_1$)
    \item Identical to case 1, but instead the human ranks $x_1$ in position $m\geq 2$, and the algorithm removes all of the items the human had ranked before it ($\rho^a_1 \ne x_1, x_1 \in \rho_{[\npresent]}^a, \rho_m^h = x_1, \rho_{[m-1]}^h \subseteq \rho_{-[\nitem-\npresent]}^a$)
\end{enumerate} 
\end{defin}

\begin{defin}\label{def:badevent}

A \enquote{bad event} is a pair of permutations $\pi^a, \pi^h$ where the joint human-algorithm system \emph{fails} to pick the best arm, where the algorithm alone would have picked it. 

A \enquote{bad event} occurs when the algorithm ranks $x_1$ first, but the human does not ($\pi_1^a = x_1, \pi_1^h \not = x_1$) and it is \emph{not} the case that the human ranks $x_1$ in position $m$, and the algorithm removes all of the items the human had ranked before it (\emph{not} that $\pi_1^a \in \pi_{\npresent}^a, \pi_m^h = x_1,\pi_{[m-1]}^h \subseteq \pi_{-[\nitem-\npresent]}^a$). 
\end{defin}

Lemma \ref{lem:bijective} states that there exists a bijective mapping between \enquote{good events} and \enquote{bad events} - that is, for every \enquote{good event} there is a unique corresponding \enquote{bad event}. As an immediate corollary, we see that there must be equal numbers of good and bad events. This result by itself is somewhat surprising, and highlights the importance of the probability distributions $\algdist, \humdist$: given a uniform distribution over permutations, the good events and bad events are equally likely, so any complementarity must be driven by certain permutations being more likely than others.

\begin{restatable}{lemma}{bijective}
\label{lem:bijective} 
For any human algorithm system with $\npresent < \nitem$, there is a bijective mapping between \enquote{good events} and \enquote{bad events}. 
\end{restatable}

\begin{corollary}
There are equal numbers of \enquote{good events} and \enquote{bad events}. 
\end{corollary}

While the full proof of Lemma \ref{lem:bijective} is deferred to Appendix \ref{app:proofs}, the relevant bijective mapping will be useful for later analysis. We define it as \enquote{best-item-mapping}, a function mapping from \enquote{good events} to \enquote{bad events} by swapping the indices of the best item $x_1$ and whichever item $x_j$ that the algorithm had ranked first instead of $x_1$.   

\begin{defin}[Best-item mapping]\label{def:bestmap}
Take any pair of permutations $\rho^a, \rho^h$ such that 
$$\rho^a_1 = x_j \quad \rho^a_i = x_1 \quad \rho^h_{m} = x_1 \quad \rho^h_{\ell}= x_j$$
for $x_j \ne x_1$. Then, we construct the new permutations $\pi^a, \pi^h$ by flipping the location of items $x_1, x_j$, keeping all other items in the same location: 
$$\pi^a_1 = x_1 \quad \pi^a_i = x_j \quad \pi^h_{m} = x_j \quad \pi^h_{\ell}= x_1$$
\end{defin}
To illustrate this mapping process, Example \ref{ex:bestmap} walks through a worked example. 

\begin{example}\label{ex:bestmap}
Consider the case where $\nitem=3, \npresent=2$, for items $x_1, x_2, x_3$ (in descending order of value). An example of a \enquote{good event} is given by the following pairs of permutations: 
$$\rho^a = [x_3, x_1, x_2] \quad \rho^h = [x_2, x_1, x_3]$$
This can be observed directly: for $\npresent=2$ the algorithm will present items $x_3, x_1$, and between these two items the human ranks item $x_1$ first, so it will ultimately pick this (best) item, while the algorithm acting by itself would not have picked this item. 

Applying the best-item mapping in Definition \ref{def:bestmap} goes as follows: the algorithm ranks item $x_3$ first, so we define $x_j = x_3$. To obtain the corresponding \enquote{bad event}, we swap the location of items $x_1, x_3$ for both the algorithm and human permutations, to obtain: 
$$\pi^a = [x_1, x_3, x_2] \quad \pi^h = [x_2, x_3, x_1]$$
We can again verify that this is a valid \enquote{bad event}: the algorithm presents items $x_1, x_3$ to the human, but the human would ultimately pick item $x_3$. However, if the algorithm were acting by itself, it would have picked the best item $x_1$. 
\end{example}

\section{Complementarity without anchoring}\label{sec:noanch}
The preliminary results for \enquote{good events} and \enquote{bad events} in the Section \ref{sec:prelim} are simply mapping between events: they hold for all probability distributions $\algdist, \humdist$. In this and the next section, we will focus on the Mallows model and give conditions such that the joint system will perform strictly worse or better than human or algorithm alone (achieving complementarity \cite{bansal2021does}). In this section, we will consider the case where the human and the algorithm have independent permutations (unanchored). Section \ref{sec:equalacc} begins by considering the case where the human and the algorithm have equal accuracy rates when acting by themselves, showing that complementarity can be guaranteed whenever exactly $\nitem=2$ items are presented. Next, Section \ref{sec:diffacc} will extend the analysis to when the accuracy rates between the human and algorithm differ, showing that the humans accuracy is more impactful, and thus the regions of complementarity are asymmetric.

\subsection{Equal accuracy in human and algorithm}\label{sec:equalacc}

First, in this section, we will give specific conditions for when complementarity is achievable in the \emph{unanchored} setting: specifically, whenever the human and algorithm have equal accuracy rates $\phi^a = \phi^h$ and the algorithm presents $\npresent=2$ items. We consider this setting particularly important because it is extremely achievable: even if the human is very bandwidth limited, it is extremely reasonable to assume that they are able to consider a finalist set of 2 items to pick between. 
\begin{restatable}{theorem}{unanchgood}
\label{thrm:unanchgood}
In the unanchored setting with permutations governed by the Mallows model, the probability of picking the best arm strictly \emph{increases} in the joint human-algorithm system when exactly 2 items are presented ($\npresent = 2$) and accuracy is identical between human and algorithm ($\phi^a = \phi^h$).  
\end{restatable}
While we will defer a full proof to Appendix \ref{app:proofs}, we will offer a proof sketch: 
\begin{proof}[Proof sketch]
This proof uses the best-item mapping from Definition \ref{def:bestmap}. In particular, we take any \enquote{good event}, apply the best-item mapping, and show that the corresponding \enquote{bad event} is equally or less likely than the corresponding \enquote{good event}. Given the Mallows model, a permutation $\pi$ is more likely if they involve fewer \emph{inversions} (instances where $i < j$ but $\pi_i < \pi_j$: a lower-valued item is ranked above a higher-valued item). 

First, we consider the algorithm's permutations.
Best-item mapping works by flipping the rank of the best item $x_1$ and $x_j$, defined as whichever item the algorithm ranked first in the \enquote{good event}. Here, we show that best-item mapping actually \emph{decreases} the total number of inversions by exactly one, making the \enquote{bad event} ordering for the algorithm strictly \emph{more} likely. Decreasing the number of inversions is the \emph{opposite} of the overall goal of this proof; the requirement that $\npresent=2$ is what upper bounds this number of inversions by exactly 1. 

However, we show that this effect is counteracted by the human's permutation Specifically, we consider each \enquote{good event} case in Definition \ref{def:goodevent} and show that best-item mapping always \emph{increases} the total number of inversions by at least one. Because the human and algorithm are assumed to have equal accuracy rates, the increase in inversions from the human's permutations cancels out the decrease in inversions from the algorithm's permutations, showing that the \enquote{bad event} is no more likely than the corresponding \enquote{good event}. 

The proof concludes by constructing an example where the \enquote{good event} is \emph{strictly} more likely than the \enquote{bad event}, showing that the total probability of \enquote{good events} is strictly more likely than the total probability of \enquote{bad events}. 
\end{proof}

Finally, we wish to comment briefly on the permutation distributions $\algdist, \humdist$. Theorem \ref{thrm:unanchgood} is specific to the Mallows model, but the proof technique relies very weakly on the Mallows assumption. Specifically, the only property that is necessary is that the best-item mapping in Definition \ref{def:bestmap} weakly increases the probability of permutations occurring. For Mallows model, this is satisfied because the probability of a permutation occurring is governed by the number of inversions present. Other probability distributions satisfying this property would show identical properties to those proven in Theorem \ref{thrm:unanchgood}\footnote{Identical reasoning would similarly extend Theorem \ref{thrm:anchbad} in Section \ref{sec:anch}.}. 

\subsection{Unequal accuracy}\label{sec:diffacc}
In the previous section, Theorem \ref{thrm:unanchgood} showed that complementarity is possible in the unanchored (independent) case with identical accuracy rates $\phi_h = \phi_a$ with $\npresent=2$. In this section, we will further explore the unanchored setting, but allowing accuracy rates to differ. Specifically, we will show that there always exist regions of complementarity: cases where a more accurate agent would strictly increase its accuracy by collaborating with a less accurate partner. However, these regions are \emph{asymmetric}: it is more likely that a more accurate human would gain from collaborating than a more accurate algorithm. 
Throughout this section, we will model the algorithm and human permutations as coming from a Mallows model. For analytical tractability, our theoretical results will focus on the case with $\nitem=3, \npresent=2$.

First, we characterize the regions of complementarity. For unequal accuracy rates, these regions are equivalent to cases where there are \emph{strict benefits} to collaborating with a less accurate partner. Lemma \ref{lem:humregion} shows that, no matter how accurate the human is, there always exists a (slightly) less accurate algorithm such that the joint system is strictly more accurate than either (achieving complementarity). 

\begin{restatable}[More accurate human]{lemma}{humregion}
\label{lem:humregion}
Consider $\nitem=3, \npresent=2$ where the human and algorithm both have unanchored Mallows models with differing accuracy parameters $\phi_a \ne \phi_h$. Then, there exists a region of complementarity where a more accurate human obtains higher accuracy when collaborating with a less accurate algorithm. 
Specifically, for all $\phi_h > 0$, so long as $\phi_a  \geq \max\p{\frac{\phi_h}{1.3}, \phi_h - 0.3}$ the joint system has better performance than either the human alone or algorithm alone.
\end{restatable}

For context, a Mallows model with $\nitem=3$ recovers the correct permutation $[x_1, x_2, x_3]$ with probability 48\% of the time with $\phi = 1$ and 57\% of the time with $\phi = 1.3$, so the regions in Lemma \ref{lem:humregion} represent moderate but meaningful differences in accuracy levels.

Next, Lemma \ref{lem:algregion} gives a corresponding result for when the algorithm is more accurate than the human. However, this region differs substantially from that in Lemma \ref{lem:humregion}: it is substantially narrower, indicating a much smaller range where complementarity is possible.

\begin{restatable}[More accurate algorithm]{lemma}{algregion}
\label{lem:algregion}
However, the roles of the human and algorithm are not symmetric: for the same setting as in Lemma \ref{lem:humregion}, the zone of complementarity is much narrower. Specifically, complementarity is possible for $\phi_a \in [\phi_h, \phi_h\cd (1 + 0.01)]$, for all $\phi_h\leq 1$, but is never possible for any $\phi_a \geq \phi_h + 0.15$ for $\phi_a \geq 1$. 
\end{restatable}

Finally, we will establish that the accuracy of the joint system is \emph{asymmetric} in the its dependence on the accuracy of the human and algorithm alone. Specifically, Lemma \ref{lem:orderhuman} shows that the performance of the joint system is always higher when the more accurate actor is the human, rather than the algorithm. For intuition for this asymmetry, consider the marginal impact of a more accurate algorithm - it will be more likely to include the best item $x_1$ among the $\npresent=2$ it presents. However, once the algorithm is sufficiently accurate, it will almost always present $x_1$, so increasing accuracy will have diminishing returns. A more accurate human will be more likely to select the best item $x_1$, given that it is presented - which will more directly make the joint human-algorithm system more accurate. 

\begin{restatable}{lemma}{orderhuman}
\label{lem:orderhuman}
Given any two sets of Mallows accuracies $\phi_1 > \phi_2$, for $\nitem =3, \npresent=2$, the joint system always has strictly higher accuracy whenever $\phi_a = \phi_1 > \phi_h = \phi_2$. 
\end{restatable}

As a direct corollary, we can see the regions of complementarity (when they exist) are asymmetric: 

\begin{corollary}
For Mallows model ($\nitem=3, \npresent=2$), the regions of complementarity are \emph{asymmetric} in human and algorithmic accuracy rates. 
\end{corollary}

This explains why the region of complementarity in Lemma \ref{lem:humregion} and Lemma \ref{lem:algregion} are larger when the human is the more accurate one - the human's accuracy more rapidly increases the accuracy of the joint system, which outperforms the more accurate actor (here, the human) for a wider range of accuracy differentials.

The results from this section are illustrated in Figure \ref{fig:diff_acc_symbolic}. The contour plot gives the accuracy of the joint human-algorithm system, which is strictly increasing in $\phi^a, \phi^h$. Overlaid in blue is the analytically derived region of complementarity. The regions derived in Lemmas \ref{lem:humregion} and \ref{lem:algregion} are overlaid in red and white, respectively. Note that the red region (where the complementarity occurs with a more accurate human) encompasses almost all of the zone of complementarity, while the white region (where complementarity occurs with a more accurate algorithm) is comparatively minuscule.

\begin{figure}
    \centering
    \ifarxiv
        \includegraphics[width=2.5in]{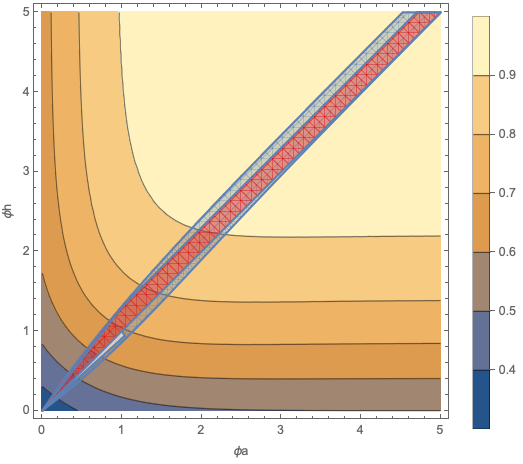}
    \else 
        \includegraphics[width=2.2in]{mallows_contour_02_23_24}
        \fi
    \caption{A plot showing relative accuracy of the joint system for differing algorithm and human accuracy $\phi_a, \phi_h$, given a Mallows model distribution for each actor with $\nitem=3, \npresent=2$.     The contour plot gives the accuracy of the joint human-algorithm system, which is strictly increasing in $\phi^a, \phi^h$. Overlaid in blue is the analytically derived region of complementarity. The regions derived in Lemmas \ref{lem:humregion} and \ref{lem:algregion} are overlaid in red and white, respectively.
    Note that this plot is symbolic and thus not based on simulations. }
    \label{fig:diff_acc_symbolic}
\end{figure}

\section{Impact of anchoring}\label{sec:anch}

In previous section, we assumed that the human and algorithm drew \emph{independent} realized permutations over items. In this section, we will relax that assumption and consider cases where the permutations may be correlated. There are a few natural settings where this might fail to hold: 
\begin{itemize}
    \item The human and the algorithm have correlated sources of information: when one makes an error, the other is also likely to make an error. 
    \item Human cognitive biases: if the human observes the algorithm's ordering before formalizing their own ordering, the human may \emph{anchor} on the algorithm's ordering and allow it to influence their own beliefs. 
\end{itemize}
In general, we will show that anchoring makes achieving complementarity difficult or impossible, depending on the degree of anchoring. As described in Section \ref{sec:model}, for the Mallows model we we model the distance of a particular permutation $\pi$ as given by the weighted average of the distance to the human's central distribution $\pi^*$ and the algorithm's realized permutation $\pi^a$: 
$$d(\pi^*, \pi^a, \pi, w_a) = (1-w_a) \cd d(\pi^*,\pi) + w_a \cd d(\pi^a, \pi)$$

First, we begin with the most extreme form of anchoring: $w_a=1$, where the human takes the algorithm's realized permutation $\pi^a$ as their \enquote{ground truth}. Theorem \ref{thrm:anchbad} below, begins by showing that when such  anchoring is present, the joint system always has strictly worse accuracy than the algorithm alone - no matter how many items are presented $\npresent$ or the relative accuracy rates of the human and algorithm $\phi^a, \phi^h$. This is a quite general impossibility result, indicating that a wide range of conditions lead to undesirable performance. 

\begin{restatable}{theorem}{anchbad}
\label{thrm:anchbad}
In the anchored setting with Mallows model distributions for permutations, the probability of picking the best arm strictly \emph{decreases} in the joint human-algorithm system, as compared to the algorithm alone. This holds for any $\npresent< \nitem$, no matter the accuracy rates for the algorithm and human $\phi^a, \phi^a$.
\end{restatable}
While we defer a full proof of Theorem \ref{thrm:anchbad} to Appendix \ref{app:proofs}, we give an informal proof sketch below: 
\begin{proof}[Proof sketch]
Similar to Theorem \ref{thrm:unanchgood}, this proof uses the best-item mapping from Definition \ref{def:bestmap}. In particular, we take any \enquote{good event}, apply the best-item mapping, and show that the corresponding \enquote{bad event} is strictly more likely than the \enquote{good event}. 

Again, given the Mallows model, a permutation $\pi$ is more likely if they involve fewer \emph{inversions} (instances where $i < j$ but $\pi_i < \pi_j$: a lower-valued item is ranked above a higher-valued item). Best-item mapping works by flipping the rank of the best item $x_1$ and $x_j$, defined as whichever item the algorithm ranked first in the \enquote{good event}. This mapping changes the relative ranking of $x_1$ and $x_j$, but also the pairwise ranking of every item that is in between $x_j$ and $x_1$. The full proof proves that this process always strictly \emph{decreases} the total number of inversions in the algorithm's ranking, relative to the \enquote{good event}. 

Next, we consider the human's permutation. Best-item mapping also flips the indices of $x_1, x_j$ in the human's permutation. However, in the anchored setting the human's distribution $\humdist$ is defined relative to the algorithm's presented permutation. Therefore, flipping the indices of $x_1, x_j$ for the algorithm is equivalent to relabeling the items, meaning that the human's \enquote{good event} permutation is exactly as likely as the human's \enquote{bad event} ordering, given the changed permutation\footnote{This same proof would hold regardless of if the human anchors on $\pi^a$ (the entire realized permutation) or $\pi^a_{[\npresent]}$ (only the $\npresent$) presented items. }. Because of this, our results hold no matter the accuracy rates of the human and algorithm $\phi^h, \phi^a$. 
\end{proof}

While Theorem \ref{thrm:anchbad}'s impossibility results are quite general, its degree of anchoring is also quite strict: $w_a=1$ requires that the human essentially ignore any extraneous information they may have. In Figure \ref{fig:mallows_anch} we relax this requirement and numerically explore settings with $w_a < 1$. In particular, this figure shows $w_a \in \{0, 0.25, 0.5, 0.75, 1\}$, with $\npresent \in [1, \nitem]$\footnote{Code is available to reproduce all simulations at  \url{https://github.com/kpdonahue/benefits_harms_joint_decision_making}.}. Points above the solid black line achieve complementarity, while lines below it do not. Note that complementarity is never achieved for $w_a =1$ (as predicted by Theorem \ref{thrm:anchbad}). Additionally, complementarity is achieved for $w_a = 0, \npresent=2$ (as predicted by Theorem \ref{thrm:unanchgood}). For intermediate $w_a$, complementarity is achieved if a) $w_a$ is relatively small and/or b) the number of items presented $\npresent$ is relatively small. Interestingly, this implies that in the presence of anchoring, it may be optimal to restrict the number of items that are presented, including potentially down to $\npresent=1$, allowing the algorithm to entirely pick the final item itself.

\begin{figure}
    \centering
    \ifarxiv
    \includegraphics[width=5in]{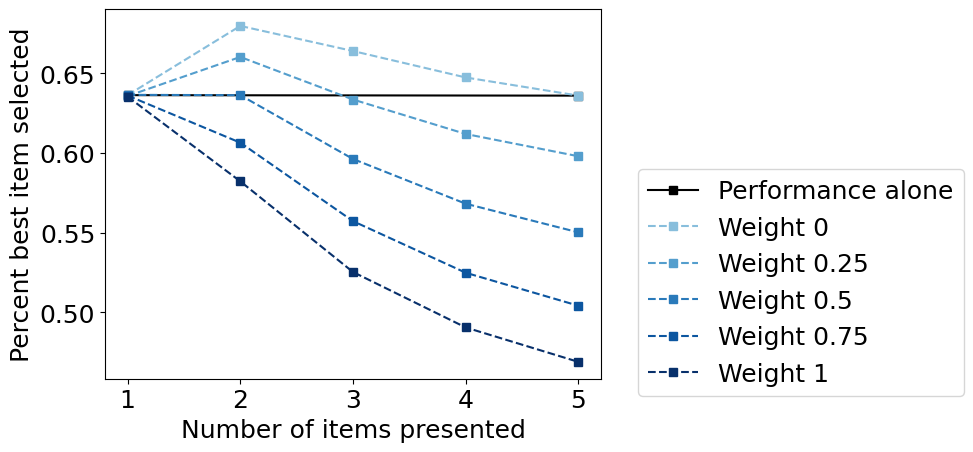}
    \else
    \includegraphics[width=3in]{plot_complete_perfect_partial_mallows_02_23_24.png}
    \fi
    \caption{$\nitem = 5$ total items, Mallows model, equal accuracy rates.  Displaying the impact of (partial) anchoring. The $x$ axis gives the number of items presented ($\npresent$): note that for $\npresent=1$ this is equivalent to the algorithm picking, while for $\npresent=\nitem=5$ this is equivalent to the human picking alone. Weight measures how strongly the human anchors on the algorithm, with $0$ representing independence and $1$ representing the strongest anchoring (as in Theorem \ref{thrm:anchbad}, where complementarity is impossible). Each point represents average of 10 trials each with $5\cd 10^4$ simulations each (error bars omitted, on the order of $0.01$). 
    }
    \label{fig:mallows_anch}
\end{figure}

\section{Random Utility Model}\label{sec:rum}
In the previous sections of this paper, we have primarily focused on giving theoretical results for the Mallows model. While the Mallows model is frequently used as a model of permutations, especially given its analytical tractability, it is not the only model of permutations. In particular, the Random Utility Model (RUM) is also frequently used (see Section \ref{sec:model} for a detailed description of these models). While the RUM is less tractable theoretically, in this section we show numerically that the main phenomena we observe with the Mallows model also translates to the RUM - suggesting that our results are robust to the exact type of permutation generation distributions used.

\subsection{Unanchored (extension of Section \ref{sec:noanch})}
First, we consider the results in Section \ref{sec:noanch}, which showed that (for the Mallows model): 
\begin{itemize}
    \item When the human and algorithm have identical accuracy rates, complementarity is guaranteed when $\npresent=2$ items are presented (Theorem \ref{thrm:unanchgood}). 
    \item When the accuracy rates differ between the human and the algorithm, the impact on overall performance is asymmetric, with the human's accuracy being more impactful (Lemma \ref{lem:orderhuman}). 
\end{itemize}
We will observe numerically that both of these main results translate to the RUM. Recall that Figure \ref{fig:diff_acc_symbolic} showed \emph{regions of complementarity} in the Mallows model: Figure \ref{fig:normal_accuracy} displays a variant of this figure in the RUM (yellow regions show complementarity, purple regions show lack of complementarity, and the red line gives the $y=x$ line of symmetry). In the RUM, greater accuracy levels are reflected by smaller standard deviations in noise, so the axes reflect varying levels of algorithmic and human noise. Within Figure \ref{fig:normal_accuracy}, we can observe: 
\begin{itemize}
    \item The red $y=x$ line is always within the yellow region of complementarity, suggesting that complementarity is always possible given equal accuracy rates and $\npresent=2$, matching our results for Mallows in Section \ref{sec:equalacc}. 
    \item The yellow region of complementarity is asymmetric: it extends further up (more accurate human) than it does to the right (more accurate algorithm), matching our results for Mallows in Section \ref{sec:diffacc}. 
\end{itemize}
Finally, note that Figure \ref{fig:normal_accuracy} also differs from Figure \ref{fig:diff_acc_symbolic} in its parameters ($\nitem=10$ as compared to $\nitem=3$). This was done deliberately to demonstrate the robustness of the results in Section \ref{sec:diffacc} to larger values of $\nitem$. 

\begin{figure}
    \centering
    \ifarxiv
    \includegraphics[width=3in]{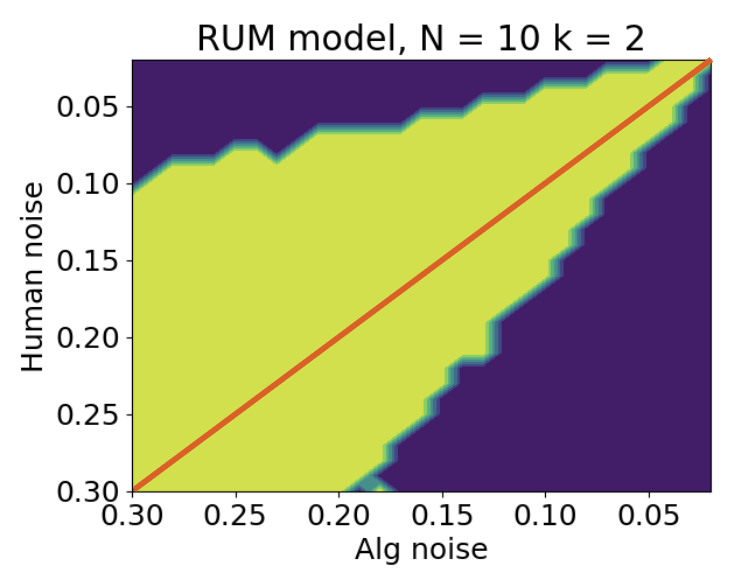}
    \else
    \includegraphics[width=3in]{diff_acc.png}
    \fi
    \caption{A version of Figure \ref{fig:diff_acc_symbolic}, but given a \emph{RUM} with \emph{Normal} distribution for each actor with $\nitem=10, \npresent=2$. Similar to Figure \ref{fig:diff_acc_symbolic}, the $x$ and $y$ axis show increasing accuracy (here, decreasing variance). For clarity, we have flipped the axes to match Figure \ref{fig:diff_acc_symbolic} so the lower left and upper right mean high noise and perfect accuracy, respectively. The yellow region is where complementarity occurs, while the purple region is where complementarity fails to occur, and the red line gives the $x=y$ axis of symmetry. 
    }
    \label{fig:normal_accuracy}
\end{figure}

\subsection{Anchoring (extension of Section \ref{sec:anch})}
Finally, in this section, we extend the anchoring results in Section \ref{sec:anch} to the Random Utility Model. As described in Section \ref{sec:model}, anchoring in the RUM is modeled such that the human draws their mean from a noise distribution with mean $w_a \cd \mu_j + (1-w_a) \cd \mu_i$, where $w_a$ is a weight parameter indicating how much the algorithm's ordering anchors the human's permutation, and $j$ is the index of item $i$ in the algorithm's permutation $\pi^a$. In this way, $w_a = 0$ reflects the unanchored case, while $w_a = 1$ reflects the anchored case. 

Figure \ref{fig:rum_anch} gives a RUM version of Figure \ref{fig:mallows_anch}, again given various weights $w_a$ and number of items presented $\npresent$. Again, we can see qualitatively similar results: complementarity is possible so long as a) $w_a$ is relatively small and/or b) the number of items presented $\npresent$ is relatively small.

\begin{figure}
    \centering
    \ifarxiv
    \includegraphics[width=5in]{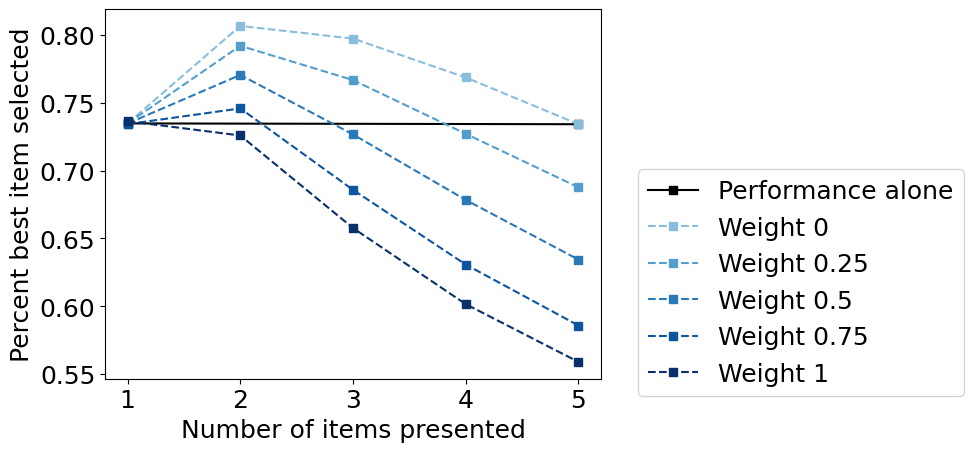}
    \else
    \includegraphics[width=3in]{plot_complete_perfect_partial_RUM_02_23_24.png}
    \fi
    \caption{$\nitem = 5$ total items, Random Utility model, equal accuracy rates.  Displaying the impact of (partial) anchoring. Version of Figure \ref{fig:mallows_anch} with the Random Utility model. The $x$ axis gives the number of items presented ($\npresent$): note that for $\npresent=1$ this is equivalent to the algorithm picking, while for $\npresent=\nitem=5$ this is equivalent to the human picking alone. Weight measures how strongly the human anchors on the algorithm, with $0$ representing independence and $1$ representing the strongest anchoring. Each point represents average of 10 trials each with $5*10^4$ simulations each (error bars omitted, on the order of $0.01$. 
    }
    \label{fig:rum_anch}
\end{figure}

\section{Discussion and future work}\label{sec:conclude}

In this paper, we have proposed a model of human-algorithm collaboration where neither the human or algorithm has ultimate say, but where they successively filter the set of $\nitem$ items down to $\npresent$ and finally a single choice. We focus on how the noise distributions $\algdist, \humdist$ influence whether the combined system has a higher chance of picking the best (correct) item. Future work extend our results to a wider range of noise model. Other interesting extensions could consider more complex models of human-algorithm collaboration - for example, cases where the human and algorithm can \enquote{vote} on the ordering of items, or other models of interaction. Additionally, they could explore cases where either the human or the algorithm is inherently biased - for example, when the algorithm has a central distribution that does that rank the best item first.

 \ifarxiv
\section{Acknowledgments}
We are extremely grateful to Kiran Tomlinson, Manish Raghavan, Manxi Wu, Aaron Tucker, Katherine Van Koevering, Kritkorn Karntikoon, Oliver Richardson, and Vasilis Syrgkanis for invaluable discussions. 
\else 
\fi 

\clearpage 
\bibliography{aaai24}

\appendix 
\section{Proofs}\label{app:proofs}

\bijective*

\begin{proof}

We use $\rho^{h}, \rho^{a}$ for the \enquote{good} orderings and $\pi^h, \pi^a$ for the \enquote{bad} orderings. We will use $\rho_{[\npresent]}^a$ to denote the $\npresent$ items that the algorithm ranks first (and thus presents to the human) and $\rho_{-[\nitem-\npresent]}^a$ to denote the $\nitem-\npresent$ items that the algorithm ranks last (and fails to present to the human). 

The \enquote{good event} occurs when in one of two cases occurs: 
\begin{enumerate}
    \item The algorithm does not rank $x_1$ first but includes it in the $\npresent$ items it presents, while the human ranks item $x_1$ first ($\rho_1^a \not = x_1, x_1 \in \rho_{[\npresent]}^a, \rho_1^h = x_1$)
    \item Identical to case 1, but instead the human ranks $x_1$ in position $m\geq 2$, and the algorithm removes all of the items the human had ranked before it ($\rho^a_1 \ne x_1, x_1 \in \rho_{[\npresent]}^a, \rho_m^h = x_1, \rho_{[m-1]}^h \subseteq \rho_{-[\nitem-\npresent]}^a$)
\end{enumerate} 

Similarly, the \enquote{bad event} occurs when the following holds: the algorithm ranks $x_1$ first, but the human does not ($\pi_1^a = x_1, \pi_1^h \not = x_1$) and it is \emph{not} the case that the human ranks $x_1$ in position $m$, and the algorithm removes all of the items the human had ranked before it (\emph{not} that $\pi_1^a \in \pi_{\npresent}^a, \pi_m^h = x_1,\pi_{[m-1]}^h \subseteq \pi_{-[\nitem-\npresent]}^a$). 

\paragraph{Mapping from good events to bad events:}
For this proof, we will begin by starting start with any pair of \enquote{good event} orderings $\rho^{h}, \rho^{a}$ and show a mapping $F$ to a pair of \enquote{bad event} orderings $\pi^h, \pi^a$. 

First, we will define $\rho^a_1 = x_j$ (the item that the algorithm ranks first in the \enquote{good event}), and index $i$ such that $\rho^a_i = x_1$ the algorithm's ranking of the best item. Similarly, we will define $m, \ell$ such that $\rho^{h}_m = x_1, \rho^{h}_{\ell} = x_j$ (the index of the item $x_1, x_j$ respectively for the human in the good event). 

Note that in the \enquote{good event}, we either have $m=1$ (the human ranks the best arm $x_1$ first), or $m\geq 2$ and $\rho_{[m-1]}^h \subseteq \rho_{-[\nitem-\npresent]}^a$. Note that we cannot have $\ell=1$ (that the human ranks item $x_j$ first): If $m=1$ this is clearly impossible because we already know $x_1 \ne x_j$ is ranked first. If $m>1$, then $\ell=1$ means that $x_j = \rho^h_1$, which by the preconditions of the \enquote{good event} must mean that $x_j$ is removed by the algorithm, or $x_j \in \rho^{a}_{-[\nitem - \npresent]}$. However, since we have defined $x_j$ as the item ranked first by the algorithm ($\rho^a_1 = x_j$), this leads to a contradiction. 

Then, we define the \enquote{bad event} through the function $F$ by swapping items $x_1, x_j$ on both sides (as in best-item mapping in Definition \ref{def:bestmap}): 
$$\pi^a_1 = x_1 \quad \pi^a_i = x_j \quad \pi^h_{m} = x_j \quad \pi^h_{\ell}= x_1$$
Note that this satisfies the preconditions of the \enquote{bad event}: the algorithm ranks $x_1$ first and the human ranks $x_1$ in position $\ell$ (which by our prior reasoning, $\ell \ne 1$). Additionally, we will show that it cannot be the case that the algorithm removes all elements that the human had ranked above the best item (\emph{not} the case that $\pi_{\ell}^h = x_1$, with $\pi_{[\ell-1]}^h \subseteq \pi_{-[\nitem-\npresent]}^a$). We will assume by contradiction that the algorithm removes all items that the human ranks above $x_1$ ($\pi_{[\ell-1]}^h \subseteq \pi_{-[\nitem-\npresent]}^a$) and show that this implies a conflict in the definition of the \enquote{bad event}. 

First, we will consider the case where $x_j$ is ranked above $x_1$ by the human ($x_j \in \pi_{[\ell-1]}^h$). Assume by contradiction that the algorithm removes all items that the human ranks above $x_1$ ($\pi_{[\ell-1]}^h \subseteq \pi_{-[\nitem-\npresent]}^a$). Then, this would imply that $x_j \in \pi_{-[\nitem-\npresent]}^a$, which implies that in the \enquote{good event} $x_1 \in \rho_{-[\nitem-\npresent]}^a$ (because items $x_1, x_j$ are swapped). However, this means that in the \enquote{good event} the algorithm removes the best arm $x_1$, again disallowed by the preconditions of the \enquote{good event}.

Next, we will consider the case where item $x_j$ is not ranked above $x_1$ by the human ($x_j \not \in \pi_{[\ell-1]}^h $). This implies that $\pi_{[\ell-1]}^h$ is unaffected by the $x_j, x_1$ swapping (since we know $x_1 = \pi_{\ell}^h$ and is also not in this set). This means that in the \enquote{good event}, all elements that the human ranked above the $\ell$th position are in the set that the algorithm fails to present ($\rho_{[\ell-1]}^h \subseteq \rho_{-[\nitem - \npresent]}^a$, with $\rho_{\ell}^h = x_j$. We know that in the \enquote{good event} the algorithm always presents item $x_j$ because by definition $\rho_1^a = x_j$. However, this implies that the human would pick $x_j$ rather than $x_1$, which violates the preconditions of the \enquote{good event}.

\paragraph{Mapping from bad events to good events:}
We will show that this mapping is a bijection by showing that there exists an inverse i.e. a function $G$ mapping from each \enquote{bad event} $\pi^a, \pi^h$ to a \enquote{good event} $\rho^a, \rho^h$. Again, this mapping will involve swapping the items $x_1, x_j$. Given $\pi^a, \pi^h$ satisfying the preconditions of the \enquote{bad event}, we know that $\pi^a_1 = x_1$ and $\pi_{\ell}^h = x_1$ for some $\ell$. Defining $x_j$ will be slightly more subtle. 

Case 1: First, we will consider the case where the human's first item is presented by the algorithm ($\pi_1^h \in \pi^a_{[\npresent]}$). Then, we will define $x_j = \pi_1^a$. Having defined $x_j$, we will additionally define the index where the algorithm ranks item $x_j$ as $i$ such that $\pi_i^a = x_j$. Then, we will construct the \enquote{good event} by swapping elements $x_1, x_j$. This satisfies the precondition of the \enquote{good event}: $\rho_1^a = x_j \ne x_1$ and by assumption of Case 1, $x_j = \rho^h_1 \in \pi^a_{[\npresent]}$, which implies $x_1 \in \rho^a_{[\npresent]}$. Additionally, by construction the human ranks the best item first $\rho^h_1 = x_1$. 

Case 2: Next, we will consider the case where the human's first item is \emph{not} presented by the algorithm ($\pi_1^h \not \in \pi^a_{[\npresent]}$). Then, we will define $x_j = \pi_m^h$ as lowest $m$ such that the algorithm presents ($\pi_m^h \in \pi^a_{[\npresent]}$). We will again define $i$ such that $\pi_i^a = x_j$. Again, we will construct the \enquote{good event} by swapping elements $x_1, x_j$. By similar reasoning, this satisfies the preconditions of the \enquote{good event}. First, by the definition of the bad event, we have $\pi^a_1 = x_1$, which means that after applying $G$ we must have $\rho_1^a = x_j \ne x_1$ and $\rho_i^a = x_1$. 
We must have that the algorithm presents the best item ($x_1 \in \rho^a_{[\npresent]}$): if this fails to hold, then we know that in the \enquote{bad event} the algorithm fails to present item $x_j$. However, by assumption of how we defined $x_j$, we required that it was the human's highest-ranked item that the algorithm presented. 

Next, we will analyze the human's ordering. In this case, by assumption $x_j = \pi^h_m$, with all lower indexed items being removed by the algorithm ( $\pi^h_{[m-1]} \subseteq \pi^a_{-[\nitem-\npresent]}$). This implies that $x_1 = \rho^h_m$, again with all lower indexed items being removed by the algorithm ($\rho^h_{[m-1]} \subseteq \rho^a_{-[\nitem-\npresent]}$). Because we have already shown that the algorithm must present the best arm $x_1$, this shows that the joint human-algorithm system must pick the best arm, when the algorithm alone would not (satisfying the conditions of the \enquote{good event}).

Finally, we can note explicitly that $F$ and $G$ are inverses of each other. Function $F$ takes good events $\rho^a, \rho^h$ and swaps the indices of items $x_j = \rho^a_1$ and $x_1$ to produce bad events $\pi^h, \pi^a$ such that $\pi^a_1 = x_1$. Function $G$ similarly swaps two items $\tilde x_j$ and $x_1$, where we will show that $\tilde x_j = x_j$ (the pairs of items swapped are identical). 
\begin{itemize}
    \item If the human's first element is presented by the algorithm in the bad event ($\pi^h_1 \in \pi^a_{[\npresent]}$), we define $\tilde x_j = \pi_1^a$ (the algorithm's highest ranked item). Note that $F$ results in moving $x_j$ to index 1 for the algorithm, and so $\tilde x_j = x_j$ in this case. 
    \item If the human's first element is not presented by the algorithm, then $G$ defines $\tilde x_j$ as the lowest-index item in $\pi^h$ such that the algorithm presents it ($\tilde x_j \in \pi^a_{[\npresent]}$). By the definition of the bad event, we know that the algorithm ranked the best item first ($\rho^a_1=x_1$) and thus after $G$ is applied we must have $\tilde x_j = x_j$, as desired. 
\end{itemize}
Having constructed a map $F$ from \enquote{good events} to \enquote{bad events} and its inverse $G$ from \enquote{bad events} to \enquote{good events}, we have created a bijection between the space of \enquote{good events} and \enquote{bad events}. 
\end{proof}

\unanchgood*

\begin{proof}
In order to prove this result, we will use the bijective best-item mapping process of flipping the indices of items $x_1, x_j = \rho^a_1$. We will show that this swapping maps from \enquote{good events} to \enquote{bad events} in a way that \emph{weakly increases} the total number of inversions. This means that, for every good event, there is exactly one bad event that is \emph{equally likely or less likely}. Additionally, we show that there exists at least \enquote{good event} for which this mapping strictly increases the total number of inversions, which implies that the the total good events are more likely to occur than the total set of bad events. As compared to Theorem \ref{thrm:anchbad} in this proof we will assume that the algorithm presents exactly 2 items to the human ($\npresent = 2$), which will be a necessary condition for our analysis. 

\paragraph{Algorithmic mapping: }
First, we will consider the algorithm's good ordering $\rho^a$ and compare it to the number of inversions in the bad ordering $\pi^a$. Recall that in the \enquote{good event}, we require that $\rho^a_1 = x_j \ne x_1$ (the algorithm does not rank the best item first) and $x_1 \in \rho^a_{[\npresent]}$ (the algorithm includes the best arm in the top $\npresent$ presented). If $\npresent=2$, this exactly requires that $\rho^a_2 = x_1$ (the algorithm ranks the best item second). In this case, the swap mapping flips the adjacent items $x_j, x_1$, which results in an increase of exactly one inversion. This means that the best-item mapping process makes the algorithm's ordering exactly one inversion \emph{more likely}. We will show that the human's ranking is at least 1 swap \emph{less} likely in order to counteract this. 

\paragraph{Human mapping: }
Next, we will consider the human's ordering in the good event $\rho^h$ and its corresponding bad ordering in the bad event $\pi^h$, which is again constructed by swapping elements $x_1, x_j$. In the \enquote{good event}, it is either the case that a) $\rho_1^h = x_1$ (the human ranks the best arm first) or b) $\rho_m^h = x_1, \rho^h_{[m-1]} \subseteq \rho^a_{-[\nitem - \npresent]}$ (the algorithm removes every element that the human had ranked above the best arm). 

First, we will consider case a). We will define $\ell$ such that $\rho_{\ell}^h = x_j$. We can consider flipping the indices of $x_1, x_j$ in two stages. First, we move $x_1$ from position 1 to position $\ell$. This takes $\ell-1$ swaps and will result in $\pi^h_{\ell-1} = x_j, \pi^h_{\ell} = x_1$. Again, because $x_1$ is the highest ranked item, each of these swaps adds an inversion, making the arrangement \emph{less} likely because it is moving $x_1$ further from its true position of $1$. Next, we move arm $x_j$ from position $\ell-1$ to position $1$. This involves $\ell-2$ swaps. Because we do not know the value of $x_j$ and the relative value of items in between them, we do not know exactly how many inversions this results in. These all \emph{may} make the arrangement more likely: for example, if $x_j = x_2$ and $\ell>2$. However, in this case, the increase in inversions from the first step ($\ell-1$) is at least one more than the decrease in inversions from the second step (no more than $\ell-2$), so the combined swap process still makes the overall setting at least 1 inversion more likely. 

Next, we will consider case b) for the human's ordering, where $\rho^h_m = x_1$, with $\rho^h_{[m-1]} \subseteq \rho^a_{-[\nitem-\npresent]}$ (the algorithm removes all items that the human had ranked before the best items). Because $\rho^a_1 = x_j$ by definition, we know that the algorithm presents item $x_j$, and therefore $\rho_{\ell}^a = x_j$ for some $\ell > m$. By a similar reasoning, we can show that flipping $x_j, x_1$ makes the arrangement at least one inversion more likely. Again, we will consider this process in two phases. First, we move item $x_1$ from position $m$ to position $\ell$. This takes $\ell-m$ swaps and results in $\pi_{\ell-1}^h = x_j, \pi_{\ell}^h = x_1$. Each of these swaps adds an inversion, because $x_1$ is the highest value item, which is being moved further from its true position in rank 1. Next, we move arm $x_j$ from position $\ell-1$ to position $m$. This involves $\ell-m-1$ swaps. Again, these swaps may reduce the number of inversions by at most $\ell-m-1$ (depending on the value of $x_j$). However, the total number of increases in inversions ($\ell-m$) is at least more than the maximum number of decreases in inversions ($\ell-m-1$). 

\paragraph{Strict increase: }
So far, we have shown that the best-item mapping from good events $\rho^a,\rho^h$ to bad events $\pi^a, \pi^y$ weakly increases the number of swaps (weakly reducing the probability of it occuring). Next, we will show that there exists at least one setting where this mapping is strict, which would imply that the \enquote{good events} have total probability that is strictly greater than the \enquote{bad events}. 

We construct this case as follows: we define $x_j = x_{\nitem} = \rho^a_1$ (the lowest-ranked item) and assume that the human ranks the best item first and the worst item last: $\rho^h_1 = x_1, \rho^h_{\nitem} = x_{\nitem}$. Because $\npresent=2$ items are presented, we know that $x_1 = \rho^a_2$. In other words, 
$$\rho^a = [x_{\nitem}, x_1, \ldots ] \quad \rho^h = [x_1, \ldots x_{\nitem}]$$
Following the swap mapping process gives us that: $\pi^a_1 = x_1, \pi^a_2 = x_{\nitem}, \pi^h_1 = x_{\nitem}, \pi^h_{\nitem} = x_1$, or: 
$$\pi^a = [x_1, x_{\nitem}, \ldots ] \quad \pi^h = [x_{\nitem}, \ldots x_1]$$
Moving from $\rho^a$ to $\pi^a$ involves adding exactly 1 inversion (making this ordering slightly more likely). Moving from $\rho^h$ to $\pi^h$ involves first moving $x_1$ from position 1 to position $\nitem$ ($\nitem-1$ swaps, each of which are an inversion and so make the arrangement less likely), and then moving $x_{\nitem}$ from position $\nitem-1$ to position $1$ ($\nitem-2$ swaps, again each of which are an inversion and thus make the arrangement less likely). In total, this involves $2 \cd \nitem-3$ inversions. This is greater than the $1$ swap involved in moving from $\rho^a$ to $\pi^a$ whenever $2 \cd \nitem -3 >1$, which occurs exactly whenever $\nitem > 2$, the minimal assumption we require. 
\end{proof}

\humregion*
\begin{proof}
First, we need to obtain closed-form solutions for the accuracy of the algorithm and human alone. The human acting alone picks the best arm whenever it arrives at the permutation $[x_1, x_2, x_3]$ (0 inversions) or $[x_1, x_3, x_2]$ (1 inversion). The probability of this occuring is given by: 
$$P_h(\phi_h) = \frac{1 + \exp(-\phi_h)}{Z_h} $$
$$Z_h = 1 + 2 \exp(-\phi_h) + 2 \exp(-2 \phi_h) + \exp(-3 \cd \phi_h)$$
where $Z_h$ is the normalizing constant giving the probability of every possible permutation of $\nitem=3$ items. The probability of the algorithm picking the best arm is identically given by: 
$$P_a(\phi_a) = \frac{1 + \exp(-\phi_a)}{Z_h} $$
$$Z_a = 1 + 2 \exp(-\phi_a) + 2 \exp(-2 \phi_a) + \exp(-3 \cd \phi_a)$$
The conditions for when the joint system picks the best arm is more complex, but can be calculated by enumerating the permutations that lead to complementarity and the number of inversions involved in each. The probability of this event occurring is given by $P_c(\phi_a, \phi_h)$ equal to $\frac{1}{Z_a \cd Z_h}$ multiplied by the quantity below:
\ifarxiv
$$ 1 + \exp(-\phi_h) + \exp(-2 \phi_h) +2 \exp(-\phi_a) + \exp(-2\phi_a) + \exp(-\phi_a - 2 \phi_h)  +3 \exp(-\phi_a - \phi_h) + 2 \exp(-2 \phi_a - \phi_h)$$
\else 
$$ 1 + \exp(-\phi_h) + \exp(-2 \phi_h) +$$ 
$$2 \exp(-\phi_a) + \exp(-2\phi_a) + \exp(-\phi_a - 2 \phi_h)$$
$$ +3 \exp(-\phi_a - \phi_h) + 2 \exp(-2 \phi_a - \phi_h)$$
\fi
The difference between the accuracy of the joint system and the human alone is given by $P_c(\phi_a, \phi_h) - P_h(\phi_h)$: 
$$ = \frac{e^{\text{$\phi $h}} \left(e^{\text{$\phi $a}+\text{$\phi $h}}+e^{2 \text{$\phi $a}}-e^{\text{$\phi $h}}-e^{2 \text{$\phi $h}}\right)}{\left(e^{\text{$\phi $a}}+e^{2 \text{$\phi $a}}+1\right) \left(e^{\text{$\phi $h}}+1\right) \left(e^{\text{$\phi $h}}+e^{2 \text{$\phi $h}}+1\right)}$$
which is positive whenever: 
\begin{equation}\label{eq:humbetter}
e^{\text{$\phi $a}+\text{$\phi $h}}+e^{2 \text{$\phi $a}}-e^{\text{$\phi $h}}-e^{2 \text{$\phi $h}}>0
\end{equation}
The joint system has higher accuracy than the algorithm alone whenever $P_c(\phi_a, \phi_h) - P_h(\phi_a) $ is positive: 
$$= -\frac{e^{\text{$\phi $a}} \left(e^{\text{$\phi $a}+\text{$\phi $h}}+e^{\text{$\phi $a}+2 \text{$\phi $h}}+e^{\text{$\phi $a}}-2 e^{2 \text{$\phi $h}}-e^{3 \text{$\phi $h}}\right)}{\left(e^{\text{$\phi $a}}+e^{2 \text{$\phi $a}}+1\right) \left(2 e^{\text{$\phi $h}}+2 e^{2 \text{$\phi $h}}+e^{3 \text{$\phi $h}}+1\right)}$$
which is positive whenever: 
\begin{equation}\label{eq:algbetter}
-e^{\phi_a + \phi_h} - e^{\phi_a +2 \phi_h} - e^{\phi_a} + 2 e^{2 \phi_h} + e^{3 \phi_h}>0
\end{equation}
Note that Equations \ref{eq:humbetter}, \ref{eq:algbetter} are \emph{not} symmetric because the roles of the human and algorithm are not symmetric. 

We will constructively produce a function describing relevant $\phi_h$ values and then prove that the joint system always has higher accuracy than either the human or algorithm alone. Specifically, we note that we can rewrite the preconditions in the statement of this lemma as: 
\begin{equation}\label{eq:cases}
\begin{cases}
    \phi_h \in (\phi_a, \phi_a \cd 1.3] \quad \phi_h \leq 1\\
    \phi_h \in (\phi_a, \phi_a + 0.3] \quad \phi_h > 1
\end{cases}
\end{equation}
The cases in Equation \ref{eq:cases} are what we will show to be sufficient conditions for complementarity to exist.

First, we can observe that whenever $\phi_h > \phi_a$, Equation \ref{eq:algbetter} will be satisfied, because this implies the following three inequalities hold: 
$$\exp(2 \phi_h) > \exp(\phi_a + \phi_h) \quad \exp(2 \phi_h) > \exp(\phi_a) $$
$$\exp(3 \phi_h) > \exp(\phi_a + 2 \phi_h)$$
This means that the joint system will always outperform the algorithm alone. 
Next, the remaining task is to find conditions where Equation \ref{eq:humbetter} is satisfied (the joint system outperforms the human alone). First, we will note that Equation \ref{eq:humbetter} is decreasing in $\phi_h$ when $\phi_h > \phi_a$: 
$$\frac{d}{d\phi_h}[\exp(\phi_a + \phi_h) + \exp(2 \phi_a) - \exp(\phi_h) - \exp(2 \phi_h)]$$
$$= \exp(\phi_a + \phi_h)  - \exp(\phi_h)-2 \exp(2\phi_h) < 0$$
Therefore, if we wish to show that Equation \ref{eq:humbetter} is positive, it suffices to show it for the maximum value of $\phi_h$ we allow. 

\textbf{Low accuracy: $\phi_a \leq 1$: }
First, we will consider the case where $\phi_a\leq 1$, which by Equation \ref{eq:cases} we will set $\phi_h \in (\phi_a, 1.3 \cd \phi_a]$. Setting $\phi_h$ to its maximum value in this range turns Equation \ref{eq:humbetter} into: 
$$\exp(2 \cd \phi_a) +\exp(2.3 \cd \phi_a) - \exp(2.6 \cd \phi_a)-\exp(1.3 \cd \phi_a)$$
where our goal is to show that this term is always positive within its range of $\phi_a \in [0, 1]$. We can show this by inspection: this term is positive at its endpoints ($\phi_a =0, \phi_a = 1$) and has exactly one point of zero derivative, where it is also positive. 

\textbf{High accuracy: $\phi_a > 1$: }
Next, we will consider the case where $\phi_a> 1$, which by Equation \ref{eq:cases} we will set $\phi_h \in (\phi_a, \phi_a + 0.3]$. Again, it suffices to show that we get complementarity in the case with $\phi_h$ set to its maximum value of $\phi_a + 0.3$. Having this substitution turns Equation \ref{eq:humbetter} gives:
$$\exp(2\cd \phi_a + 0.3) + \exp(2 \cd \phi_a) - \exp(\phi_a + 0.3) - \exp(2 \cd \phi_a + 0.6) > 0$$
Pulling out common terms gives: 
$$\exp(\phi_a) \cd (-\exp(0.3) + \exp(\phi_a) \cd (\exp(0.3) + 1 +\exp(0.6))$$
This term is increasing in $\phi_a$ and is positive for $\phi_a = 1$, again showing that the the overall term is always positive. 

This again shows that Equation \ref{eq:humbetter} is satisfied in these conditions, meaning that the joint system has higher accuracy than either the human alone or algorithm alone. 
\end{proof}

\algregion*

\begin{proof}
First, we will show that there does exist a (narrow) zone of complementarity where the algorithm is slightly more accurate than the human, but still strictly benefits from collaborating with it. We will set $\phi_a \in [\phi, \phi_h \cd 1.1]$ for $\phi_h \leq 1$ and use the functional form for when complementarity is achieved from Lemma \ref{lem:humregion} Substituting in to Equation \ref{eq:algbetter} gives that complementarity occurs whenever the below term is positive: 
$$2 \cd \exp(2 \phi_h) + \exp(3 \cd \phi_h) - \exp(1.1 \cd \phi_h) - \exp(2.1 \phi_h) - \exp(3.1 \cd \phi_h)$$
This is positive whenever $\phi_h \leq 1$, indicating higher accuracy than the algorithm alone. Because $\phi_a>\phi_h$, this also shows greater accuracy than the human alone, indicating complementarity. 

Next, we will show that the region of complementarity is narrow and asymmetric. Specifically, we will show that for any $\phi_a \geq \phi_h + 0.15$ for $\phi_h \geq 1$ fails to lead to complementarity. Note that if the values of $\phi_a, \phi_h$ were reversed, this would fall within the region of complementarity from Lemma \ref{lem:humregion}.

We can use the functional form derived in Lemma \ref{lem:humregion}: Substituting in to Equation \ref{eq:algbetter} gives: 
$$\exp(2\phi_a + 0.15) + \exp(2 \phi_a) - \exp(\phi_a + 0.15) - \exp(2 \phi_a + 0.3)$$
Simplifying gives: 
$$\exp(\phi_a) \cd (-\exp(0.15) + \exp(\phi_a)(1 + \exp(0.15) + \exp(0.3))$$
which is negative so long as $\phi_a \geq 1$, as desired. 
\end{proof}

\orderhuman*

\begin{proof}
We can prove this by inspecting 
$$P_c(\phi_a, \phi_h) - P_c(\phi_h, \phi_a)$$
Dropping the common denominator of $Z_a \cd Z_h$, this difference simplifies to: 
$$\exp(-2(\phi_a + \phi_h)) \cd (-\exp(\phi_a) + \exp(\phi_h) - \exp(2 \phi_a + \phi_h) + \exp(\phi_a + 2 \phi_h))$$
This term is positive exactly whenever $\phi_h > \phi_a$, which means that if one actor has has higher accuracy, the joint system has accuracy that is optimized by placing them second (e.g. in the human's role). 
\end{proof}

\anchbad*

\begin{proof}
In order to prove this result, we will use the best-item mapping definition from Definition \ref{def:bestmap}. Specifically, we will show that this mapping maps connects any \enquote{good event} to a corresponding \enquote{bad event} in a way that strictly \emph{decreases} the total number of inversions between the human and algorithm's rankings. Because under the Mallows model, events are \emph{more} likely if they involve \emph{fewer} inversions, this implies that the the \enquote{good event} is strictly less likely than the corresponding \enquote{bad event}. Because Lemma \ref{lem:bijective} has shown that best-item mapping is bijective, this means that, for every good event, there is exactly one bad event that is more likely, implying that the total good events are less likely to occur than the total set of bad events. 

We will being with an arbitrary good event with algorithmic and human permutation $\rho^a, \rho^h$ respectively. We will define $x_j = \rho^a_1$, and $i$ such that $\rho^a_i = x_1$. By Definition \ref{def:bestmap} best-item mapping works by flipping the index of $x_1, x_j$ for both the human and algorithm. We will refer to flipping the ordering of a pair of \emph{adjacent} items $\rho_i, \rho_{i+1}$ as a single \enquote{swap}. This may increase or decrease the number of \emph{inversions} depending on whether or not $\rho_i > \rho_{i+1}$ (the item ranked in the $i$th place has higher true value than item ranked $i+1$th).  

First, we will analyze the algorithm's ordering $\rho^a$ and will show that swapping $x_1, x_j$ to obtain $\pi^a$ only decreases the number of inversions, as compared with $\rho^a$. To bring arm $x_1$ from index $i$ to $1$ involves $i-1$ swaps. Because $x_1$ is the highest value item, all of these swaps must reduce the number of inversions, which makes the arrangement more likely. At the end of this process, arm $x_1$ will be ranked first and $x_j$ will be in ranked second. Next, to bring arm $x_j$ to position $i$ will involve $i-2$ swaps. How many inversions this involves depends on the value of $x_j$. In the worst case, these could all make the arrangement less likely - for example, if $x_j = x_2$ the second highest ranked item and $i >2$. This leads to an upper bound of $i-2$ swaps that could increase the total number of inversions. However, we know that the total number of swaps that reduce inversions ($i-1$) is greater than the total maximum number of swaps that increase inversions ($i-2$), so the entire process makes $\pi^a$ the \enquote{bad event} more likely than $\rho^{a}$ the \enquote{good event}. 

Next, we will analyze the human's ordering $\pi^h$, which is again constructed by swapping elements $x_1, x_j$. In the anchored ordering, the human's ordering is \enquote{anchored} at the algorithm's permutation. This means that inversions are determined by the algorithm's realized ordering - for example, if the algorithm presents $[x_2, x_1, x_3]$ then the human would consider $x_2 >x_1$. Because arms $x_1, x_j$ have been swapped for both the human and the algorithm, this means that (from the human's perspective), this is equivalent to simply re-labeling arms $x_1, x_j$. This exactly preserves the total number of inversions, which means that $\rho^h, \pi^h$ are equally likely to occur. 

Taken together, we have constructed a bijection between the \enquote{good} and \enquote{bad} events, showing that for each good event, there exists a bad event that is strictly more likely, so the probability of picking the best item is strictly less likely in the anchored setting. 
\end{proof}

\end{document}